\documentclass[11pt,a4paper]{article}

\usepackage[T1]{fontenc}
\usepackage[utf8]{inputenc}
\usepackage[english]{babel}
\usepackage{lmodern}
\usepackage{microtype}
\usepackage{amsmath,amssymb,amsthm,mathtools}
\usepackage{booktabs}
\usepackage{tabularx}
\usepackage{enumitem}
\usepackage[a4paper,margin=1in]{geometry}
\usepackage[numbers,sort&compress]{natbib}
\usepackage[hidelinks]{hyperref}
\hypersetup{
    pdftitle={Localising Errors in Language Model Inference},
    pdfauthor={Magnus Boman}
}

\newtheorem{definition}{Definition}
\newtheorem{proposition}{Proposition}
\newtheorem{remark}{Remark}

\newcommand{\Tape}[1]{\textsf{Tape#1}}
\newcommand{\Vocab}{\mathcal{V}}
\newcommand{\Bytes}{\Sigma_{\mathrm{byte}}}
\newcommand{\Model}{F_{\theta}}
\newcolumntype{Y}{>{\raggedright\arraybackslash}X}
\newcommand{\etal}{\textit{et al.}}

\title{Localising Errors in Language Model Inference}
\author{Magnus Boman\\
\small Karolinska Institutet, Stockholm, Sweden\\
\small \texttt{magnus.boman@ki.se}}
\date{}

\begin{document}
\maketitle

\begin{abstract}
Failures of large language models are often attributed to ``the model'' as a whole, although inference comprises several transformations with different state and interfaces. We present a multi-tape operational abstraction that separates raw input, tokenisation, model parameters, intermediate computation, token scores, autoregressive token history, and rendered output. Simulation of a fixed finite-precision implementation by a Turing machine follows from standard computability assumptions. Its purpose is to identify intervention points and formulate stage-specific, falsifiable hypotheses. We revisit character counting and bounded hierarchical dependencies to distinguish information preservation from computational accessibility, memorisation from length generalisation, and architectural capacity from the behaviour of a trained model. We also correct a common scratchpad interpretation of chain-of-thought: intermediate reasoning is appended to the token history on which later forward passes condition, while character output is only its rendered form. The resulting semantic representation supports counterfactual tests based on alternative encodings, internal-state interventions, decoding controls, and trace injection without making unsupported impossibility claims.
\end{abstract}

\section{Introduction}

An autoregressive language-model response is produced by a pipeline rather than by a single undifferentiated operation. A character or byte string is tokenised, token IDs are mapped to contextual representations, scores over the vocabulary are computed, a decoding policy selects a token, and the selected token is appended to the context before the next step. Failures observed at the output may therefore be sensitive to the input representation, learned computation, decoding policy, or feedback through the generated history.
This paper introduces a multi-tape operational abstraction of that pipeline. The tapes are bookkeeping devices that separate persistent state and interfaces. The value sought here is diagnostic. By defining admissible interventions at selected interfaces, the abstraction turns informal attributions such as ``a tokenisation error'' or ``a reasoning failure'' into hypotheses that can be tested by counterfactual comparisons.
This distinction is necessary because theoretical results about transformers depend strongly on modelling assumptions. Hard \emph{versus} soft attention, numerical precision, positional encodings, architectural depth, context growth, and the number of generated scratchpad tokens can all change formal conclusions \citep{strobl2024formal}. Results about theoretical capacity also do not establish what a trained model has learned, just as behavioural failure does not by itself establish architectural impossibility. Our contributions are:
\begin{itemize}[leftmargin=*,nosep]
    \item a seven-tape macro-semantics for a fixed autoregressive inference pipeline, with an explicit separation between elementary Turing-machine transitions and higher-level implementation steps;
    \item an intervention-based definition of \emph{stage sensitivity}, which treats localisation as counterfactual evidence rather than as a consequence of how the simulator is written;
    \item analyses of character counting and bounded hierarchical dependencies, including testable hypotheses that do not assume an explicit loop or stack;
    \item an account of chain-of-thought as feedback through generated tokens in the autoregressive context, together with its resource-dependent relationship to computational expressivity.
\end{itemize}

\section{Related Work}

Transformers have been studied under several formal regimes. P\'{e}rez \etal\ prove Turing completeness for a hard-attention construction with arbitrary-precision rational activations \citep{perez2021attention}. Hahn proves limitations for fixed-depth self-attention models under specified assumptions, including limitations on hierarchical formal languages \citep{hahn2020theoretical}. Yao \etal\ show that self-attention networks can recognise and generate bounded-depth Dyck languages under explicit assumptions about depth, positional encodings, precision, and attention type \citep{yao2021bounded}. These findings concern different model classes and resource regimes, as systematised by Strobl \etal\ \citep{strobl2024formal}.

Autoregressive generation introduces another resource dimension. Merrill and Sabharwal show that intermediate generation can increase a decoder-only transformer's expressive power, with the increase depending on scratchpad length and architectural assumptions \citep{merrill2024expressive}. Schuurmans \etal\ obtain computational universality for an extended autoregressive process in which a bounded context advances over an arbitrarily long generated sequence \citep{schuurmans2024autoregressive}. These results make blanket claims about what ``transformers'' can or cannot compute unhelpful unless the resource model is stated.

Programming abstractions such as RASP characterise transformer computations in a compact formal language \citep{weiss2021thinking}. Work on learning automata shows that transformers may realise finite-state computations through shallow, non-recurrent shortcuts rather than by reproducing the apparent step structure of the target algorithm \citep{liu2023transformers}. This is one reason that the absence of a named loop or stack in an architectural description cannot establish absence of a functionally equivalent computation.

Character-level errors have been associated with tokenisation, training data, and counting complexity. In a multi-model study, Fu \etal\ found little effect of word and token frequency but stronger associations with string length, token count, and repeated occurrences of a letter \citep{fu2024large}. Such evidence motivates controlled comparisons rather than a single-cause explanation.

Mechanistic interpretability seeks to identify computations in the parameters and activations of a trained model \citep{elhage2021mathematical}. Causal-abstraction work formalises relations between high-level descriptions and lower-level mechanisms \citep{geiger2025causal}. Attribution-graph studies combine learned replacement models with interventions in an original model, while emphasising that the resulting explanations are partial \citep{lindsey2025biology}. Our abstraction is coarser. It identifies system-level state and intervention sites.

\section{Operational Abstraction}

Let \(\Bytes\) be a finite byte alphabet and let \(\Vocab=\{1,\ldots,N\}\) be a finite token vocabulary. A tokeniser is a computable map
\[
    \tau:\Bytes^{*}\rightarrow\Vocab^{*}.
\]
We write \(\kappa:\Vocab^{*}\rightarrow\Bytes^{*}\) for the corresponding decoding operation. We do not identify \(\kappa\) with \(\tau^{-1}\): normalisation, special tokens, and non-canonical token sequences may prevent tokenisation from being one-to-one, even when \(\kappa(\tau(x))=x\) for the supported inputs of interest.

Fix a model implementation \(\Model\), its parameters, a maximum accessible context \(L\), and a finite numerical format \(\mathbb{F}_{b}\). At generation step \(t\), the forward computation maps the accessible suffix of the token history to a score vector
\[
    \ell_t=\Model\!\left(\operatorname{context}_{L}(z_{1:t})\right)
    \in \mathbb{F}_{b}^{N}.
\]
Implementation-specific intermediate activations and any key--value cache are treated as work state. A decoding policy \(\sigma\) maps \(\ell_t\), and optionally random bits \(r_t\), to the next token \(z_{t+1}\). Greedy decoding is the deterministic special case \(z_{t+1}=\arg\max_i \ell_{t,i}\). The abstraction separates the following stores:
\begin{center}
\begin{tabularx}{\textwidth}{@{}l l Y@{}}
\toprule
Store & Short name & Contents and role \\
\midrule
\Tape{1} & Raw input & Prompt bytes or characters before tokenisation \\
\Tape{2} & Token history & Prompt tokens followed by every generated token; this is the history available to later forward passes, subject to the context policy \\
\Tape{3} & Codec & Tokenisation rules, vocabulary, and decoding map \\
\Tape{4} & Model specification & Architecture, parameters, numerical format, and fixed inference configuration \\
\Tape{5} & Work state & Layer activations, temporary values, and any key--value cache \\
\Tape{6} & Scores and selection & Logits or probabilities, decoding-policy state, and the selected next-token ID \\
\Tape{7} & Rendered output & Human-readable decoding of the generated-token suffix \\
\bottomrule
\end{tabularx}
\end{center}

The partition is not unique. One could merge stores, split the work state by layer, or add audit logs without changing the represented computation. Seven stores are used because they align with interfaces that can, in principle, be observed or intervened on in an implementation. The pipeline can be written as the following macro-operations:
\begin{align*}
\textsc{Tokenise}:&\quad T_2 \leftarrow \tau(T_1),\\
\textsc{Forward}:&\quad (T_5,T_6) \leftarrow \operatorname{Eval}_{T_4}(\operatorname{context}_{L}(T_2),T_5),\\
\textsc{Select}:&\quad z \leftarrow \sigma(T_6;r_t),\\
\textsc{Append}:&\quad T_2 \leftarrow T_2\,z,\\
\textsc{Render}:&\quad T_7 \leftarrow \operatorname{decode\_generated}(T_2,T_3).
\end{align*}
The final four operations repeat until an end token, stopping rule, or generation limit is reached.
These are deliberately \emph{macro-steps}. A standard multi-tape Turing machine reads and writes one symbol per tape per elementary transition; it cannot evaluate a whole softmax or append a multi-symbol token encoding in one transition. Because the codec, finite-precision arithmetic, and decoding policy are computable digital procedures, each macro can be expanded into elementary transitions.

\begin{remark}
The abstraction separates sequential simulation from transformer parallelism. Self-attention within a forward pass may process positions in parallel, while autoregressive generation remains sequential across emitted tokens. Sequential simulation changes efficiency, not the input--output function being represented.
\end{remark}

\section{Intervention-Based Error Analysis}

Let a task instance consist of input \(x\), an admissible-answer set \(A(x)\), and an error indicator \(E(y,x)=\mathbf{1}[y\notin A(x)]\). An intervention \(I_j\) replaces or modifies the contents or transformation associated with store \(T_j\), after which the remaining pipeline is executed. Interventions must be operationally well formed.

\begin{definition}[Stage sensitivity]
For an item distribution \(D\), a system is sensitive to intervention \(I_j\) at stage \(j\) when
\[
\Delta_j=
\mathbb{E}_{x\sim D}\!\left[E(Y,x)\right]
-
\mathbb{E}_{x\sim D}\!\left[E(Y^{I_j},x)\right]
\neq 0,
\]
where \(Y\) and \(Y^{I_j}\) are outputs under the baseline and intervention conditions. Random decoding additionally requires expectations over controlled random seeds or repeated samples.
\end{definition}

A positive \(\Delta_j\) shows that the intervention reduces error; it does not show that stage \(j\) is the unique root cause. Pipeline components interact, and an upstream intervention necessarily changes downstream states. The appropriate conclusion is therefore sensitivity under a specified intervention. Four useful intervention families are:
\begin{center}
\begin{tabularx}{\textwidth}{@{}l Y Y@{}}
\toprule
Target & Example intervention & Evidential scope \\
\midrule
Input and tokenisation & Compare controlled encodings; where supported, supply distinct valid token sequences that decode to the same byte string & Tests sensitivity to representation, not information loss by itself \\
Internal computation & Patch or ablate a defined activation at a named layer and position in an open-weight model & Can test a mechanistic hypothesis if the intervention target is independently identified \\
Selection & Hold logits fixed while changing greedy, sampling, or constrained decoding & Separates scoring from selection effects \\
Autoregressive trace & Insert, remove, or corrupt intermediate generated tokens before the final answer & Tests whether later computation uses the public token trace \\
\bottomrule
\end{tabularx}
\end{center}

\section{Character Counting}

Consider the task of returning the number of occurrences of an ASCII character \(a\) in a string \(x\). Tokenisation changes the units supplied to the model, but a losslessly decoded subword representation need not discard the relevant information.

\begin{proposition}[Additive character count]
Suppose \(\kappa\) is concatenative on a token sequence \(z_{1:n}\), so that
\(
\kappa(z_{1:n})=\kappa(z_1)\cdots\kappa(z_n).
\)
For any ASCII character \(a\),
\[
    \#_{a}\!\left(\kappa(z_{1:n})\right)
    =\sum_{i=1}^{n}\#_{a}\!\left(\kappa(z_i)\right).
\]
\end{proposition}

\begin{proof}
Occurrences of \(a\) in a concatenation are partitioned by the concatenated substrings. Summing their counts therefore gives the count in the complete decoded string.
\end{proof}

Because \(\Vocab\) is finite, the values \(\#_{a}(\kappa(z))\) can in principle be stored as a finite lookup table for every token. A model could then aggregate them across token positions. This observation shows that token-internal character positions need not be exposed as separate sequence positions for exact counting to be representable. The useful empirical question is consequently which controlled factors predict error and which interventions change it. The multi-tape abstraction suggests three hypotheses:
\begin{description}[leftmargin=2.5em,style=nextline]
    \item[H1 Representation sensitivity.] Holding the decoded target string and task instruction fixed where possible, alternative valid token segmentations change accuracy.
    \item[H2 Counting-complexity sensitivity.] Accuracy decreases with string length or repeated target occurrences after controlling for token count and lexical familiarity.
    \item[H3 Trace sensitivity.] Supplying a correct character decomposition in the token history improves the final count, while controlled corruption of that decomposition degrades it.
\end{description}

A suitable evaluation would combine real words with generated nonce strings, balance target counts and lengths, record the actual token IDs for each model, and use paired conditions on every item. Character-separated prompts are feasible but change the input string; they should therefore be described as representational aids, not as the same tokenisation. Directly supplying non-canonical token IDs is a stronger intervention only where the model interface permits it and the resulting sequence decodes to the same bytes. Closed and open-weight models should be reported separately, with exact model versions and decoding settings.

\section{Bounded Hierarchical Dependencies}

Centre embedding creates nested dependencies, but nesting of the kind illustrated by balanced brackets is context-free, not inherently context-sensitive. Dyck languages provide a standard formalisation and can be recognised by a pushdown automaton. For bounded nesting depth, self-attention constructions can recognise or generate the corresponding restricted languages under explicit resource assumptions \citep{yao2021bounded}. Hahn's limitations and these positive results concern different regimes \citep{hahn2020theoretical}. The repeated occurrence of a token such as \emph{that} does not itself create an ambiguity unavailable to a transformer. Positions containing the same lexical token acquire different contextual representations. The substantive questions concern whether a fixed trained system tracks the relevant dependencies and whether it generalises beyond depths or templates encountered during training. Open-ended completion is a poor test because many grammatical main verbs may complete a sentence. A controlled evaluation can instead compare the probability of matched and mismatched agreement continuations. For example, after a prefix such as
\begin{quote}
\emph{The report that the senators that the journalist interviewed criticised \ldots}
\end{quote}
one can compare \emph{was} with \emph{were}, while systematically varying the number features of each noun and the embedding depth. Matched lexical templates, singular/plural counterbalancing, and held-out depths help separate hierarchical generalisation from surface-frequency cues. The operational semantics yields hypotheses rather than a proof of inability:
\begin{itemize}[leftmargin=*,nosep]
    \item error may increase with depth even within the context window;
    \item an explicit structural trace appended to \Tape{2} may improve agreement decisions;
    \item corrupting the trace at a particular dependency may selectively reverse that improvement;
    \item behaviour may vary with positional encoding, model scale, instruction tuning, and decoding interface.
\end{itemize}
Evidence for these hypotheses would demonstrate resource- or training-dependent sensitivity.

\section{Chain-of-Thought as Autoregressive Feedback}

Chain-of-thought prompting encourages a model to emit intermediate tokens before a final answer \citep{wei2022chain,kojima2022large}. In the present abstraction, a generated reasoning token is selected from \Tape{6}, appended to \Tape{2}, represented in subsequent work state on \Tape{5}, and rendered on \Tape{7}. Later forward passes condition on the token history on \Tape{2}; they do not read the characters on \Tape{7}. Chain-of-thought is therefore an externalisation from hidden computation into a public, autoregressively reusable token sequence.

This account makes several limits explicit without turning them into impossibility claims. The trace consumes context and generation steps; its content may be incorrect or causally unfaithful; and later computation may ignore it. Conversely, generated tokens can provide additional computational steps and persistent discrete state. Theoretical work therefore finds that scratchpad generation can increase expressive power in a resource-dependent manner \citep{merrill2024expressive}. Under a different extended-decoding construction, autoregressive generation can even support universal computation \citep{schuurmans2024autoregressive}. It is consequently incorrect to say without qualification that chain-of-thought leaves the computational class unchanged.

Trace interventions make the operational claim testable. For a fixed task item, one may compare a direct answer with (i) a model-generated trace, (ii) an injected correct trace, (iii) a minimally corrupted trace, and (iv) a length-matched irrelevant trace. Differences in final accuracy or logits indicate how the public token history affects the answer.

\section{Discussion}

The multi-tape representation cannot by itself determine whether a model uses memorisation, an exact algorithm, a distributed heuristic, or several redundant strategies. Such claims require behavioural controls, internal measurements, and interventions appropriate to the desired level of explanation. Stage boundaries are modelling choices: tokenisation, embedding lookup, attention, feed-forward computation, caching, scoring, sampling, and rendering can be split more finely or grouped differently. The proposed partition is useful only to the extent that it supports reproducible measurements and interventions. The semantics also abstract away hardware scheduling and most numerical details. A faithful simulator would encode floating-point formats, kernels, parallel reductions, and implementation-specific non-determinism. Those details matter for bit-level reproduction but are not needed to define the higher-level intervention sites used here. Finally, no empirical results are reported, but experiments are under way. The character-counting and hierarchical-dependency sections specify hypotheses and controls for future evaluation. A multi-tape description of language-model inference separates state and interfaces so that claims about failure can be expressed as counterfactual, stage-specific hypotheses. Its scientific value will ultimately depend on whether those distinctions support reproducible empirical discoveries.

\section*{Acknowledgements}
We thank Sol Erika Boman, Daniel F P\'{e}rez-Ram\'{i}rez and Giacomo Verardo for important comments on an earlier draft. This research was financed by the Swedish Research Council through the project \emph{Scalable Federated Learning}.

\begingroup
\small
\raggedright
\bibliographystyle{plainnat-mab}
\bibliography{references_v3}
\endgroup

\end{document}